\documentclass[11pt]{article}

\usepackage[preprint]{acl}

\usepackage{times}
\usepackage{latexsym}

\usepackage{amsmath}
\usepackage{amssymb}
\usepackage{mathtools}
\usepackage{amsthm}
\usepackage{booktabs}
\usepackage[table]{xcolor}
\usepackage{algorithm}
\usepackage{algorithmic}
\usepackage{subcaption}
\usepackage{multirow}
\usepackage{xcolor}
\usepackage[normalem]{ulem}
\usepackage{adjustbox}

\newtheorem{theorem}{Theorem}
\newtheorem{proposition}{Proposition}[section]

\usepackage[T1]{fontenc}

\usepackage[utf8]{inputenc}

\usepackage{microtype}

\usepackage{inconsolata}

\usepackage{graphicx}

\title{Selective Successor Eligibility Traces for RLVR}

\author{
Chaoli Mou$^{1}$ \quad
Xinning Chen$^{1}$ \quad
Zhan Zhuang$^{1,2}$ \\
\textbf{Qinghe Zheng}$^{1}$ \quad
\textbf{Deyuan Kong}$^{1}$ \quad
\textbf{Yu Zhang}$^{1}$\thanks{Corresponding author.} \\
$^{1}$Southern University of Science and Technology \\
$^{2}$City University of Hong Kong \\
\texttt{12432691@mail.sustech.edu.cn} \quad
\texttt{chenxn@sustech.edu.cn} \\
\texttt{12250063@mail.sustech.edu.cn} \quad
\texttt{12532603@mail.sustech.edu.cn} \\
\texttt{12532571@mail.sustech.edu.cn} \quad
\texttt{yu.zhang.ust@gmail.com}
}

\begin{document}
\maketitle

\begin{abstract}
Reinforcement Learning with Verifiable Rewards (RLVR) has become a key approach for improving the reasoning abilities of large language models. 
However, widely used critic-free algorithms such as Group Relative Policy Optimization (GRPO) suffer from coarse-grained temporal credit assignment, as they assign the same trajectory-level advantage to all tokens and fail to distinguish the contributions of different reasoning steps. To address this challenge, we propose Selective Successor Eligibility Traces (S-trace), an critic-free eligibility trace method that enables more effective credit assignment. S-trace introduces successor importance reweighting to recover downstream credit signals under PPO-style clipping and propagates them to selected historical tokens through entropy-based token selection.
Theoretically, we establish an upper bound on its discrepancy from the raw eligibility traces-based policy gradient estimator, jointly controlled by the KL divergence and the trace decay factor $\lambda$. 
Experiments show that S-trace achieves absolute gains in average pass@16 of 1.80\% and 6.17\% over GRPO, and 2.32\% and 0.87\% over trace-style GRPO($\lambda$), on Qwen3-1.7B and Qwen3-4B, respectively, while maintaining robust gains on Qwen3-8B and producing shorter responses.
\end{abstract}

\section{Introduction}
Reasoning-oriented post-training has increasingly moved beyond Reinforcement Learning from Human Feedback (RLHF) toward Reinforcement Learning with Verifiable Rewards (RLVR), where supervision comes from objective outcome-level feedback. Unlike RLHF, which relies on subjective human preference labels \cite{ouyang2022training, rafailov2023direct, casper2023open}, RLVR is particularly well suited to mathematical and logical reasoning, where final answers can often be automatically verified. Recent frameworks, such as DeepSeek-R1 \cite{guo2025deepseek}, demonstrate the effectiveness of this paradigm, showing that verifiable rewards can elicit strong reasoning behaviors in large language models (LLMs). In particular, DeepSeek-R1 exhibits an emergent ``aha moment'', where the model reflects on and revises its reasoning during generation, suggesting that RLVR may encourage forms of model-free planning \cite{bush2025interpreting,hanna2025when}.

Despite the success of RLVR, effectively assigning credit to individual tokens in long reasoning trajectories remains challenging. 
Critic-free algorithms such as Group Relative Policy Optimization (GRPO)~\cite{shao2024deepseekmath} improve training efficiency by avoiding value estimation. However, this critic-free design leads to coarser credit assignment. GRPO relies on a ``uniform credit assignment'' assumption, assigning the same outcome-level advantage to every token in a response without explicitly distinguishing their temporal contributions. For long mathematical reasoning traces spanning thousands of tokens, such uniform assignment makes routine tokens and critical reasoning steps indistinguishable, limiting the model's ability to exploit useful temporal dependencies.

To capture temporal dependencies overlooked by uniform credit assignment, we draw inspiration from eligibility traces \cite{sutton1988learning,sutton1998reinforcement,van2021expected,gupta2024past}. We analyze their application to RLVR by reformulating the unclipped PPO surrogate gradient as temporally accumulated policy gradients under the actor-only eligibility traces framework \cite{10.5555/645527.657471}. Although eligibility traces have recently been explored in RLVR, such as GRPO($\lambda$) \cite{parthasarathi2025grpo}, we find that directly adopting raw weight-style GRPO($\lambda$) estimator fails to fully utilize valid downstream signals for temporal credit assignment under policy optimization with PPO clipping. Under PPO-style clipping, suppressing a historical token's importance-weighted gradient also suppresses its accumulated trace contribution, even when subsequent tokens retain valid learning signals. Moreover, our empirical analysis reveals that the trace-style variant of GRPO($\lambda$) exhibits higher clipping fractions than standard GRPO, further indicating that naive eligibility traces may fail to exploit available credit signals.


In this paper, we propose Selective Successor Eligibility Traces (S-trace), a critic-free eligibility trace framework for RLVR that enables more effective temporal credit assignment. S-trace introduces successor importance reweighting, which incorporates the importance ratios of subsequent tokens into historical gradient contributions, allowing valid downstream signals to propagate to earlier tokens even when their direct gradient contributions are suppressed by PPO-style clipping. We theoretically characterize the discrepancy between the successor importance-reweighted estimator and the raw eligibility trace estimator through an upper bound jointly governed by the policy KL divergence and the trace decay factor $\lambda$. Yet increasing $\lambda$ to preserve long-horizon credit assignment loosens the discrepancy bound, while dense successor importance weights may introduce noisy updates. To mitigate noisy updates caused by dense eligibility traces, S-trace regulates the propagation of successor-derived credit to selected historical tokens, with $\rho$ controlling trace sparsity. Following the observation that high-entropy tokens can indicate informative reasoning steps~\cite{wang2025beyond}, We adopt entropy-based token selection as a practical instantiation of selective credit assignment, identifying historical tokens that receive successor-derived credit within each rollout.

To summarize, our contributions are three-fold.
\begin{itemize}
\item We introduce S-trace, a critic-free eligibility trace method for RLVR that propagates surviving downstream learning signals to selected historical tokens, including those whose direct gradients are suppressed by clipping.
\item We establish an upper bound on the discrepancy between successor importance-reweighted and raw eligibility trace estimators, governed jointly by the policy KL divergence and the trace decay factor.
\item S-trace achieves absolute average pass@16 gains of 1.80\% and 6.17\% over GRPO, and 2.32\% and 0.87\% over trace-style GRPO($\lambda$), on Qwen3-1.7B and Qwen3-4B, respectively, with gains extending to Qwen3-8B alongside shorter responses. Our comparison with GSPO further offers a preliminary perspective on uniform versus recency-based credit assignment in mathematical reasoning.
\end{itemize}

\section{Related Work}
\textbf{Critic-free RL}. 
Driven by the need to reduce memory overhead and enhance computational efficiency, critic-free RL has become a dominant paradigm in LLM fine-tuning. 
Before the prevalence of group-based methods, representative on-policy approaches such as ReMax \cite{li2024remax}, and RLOO \cite{ahmadian-etal-2024-back} sought to eliminate the critic by leveraging simplistic baseline estimates. Subsequently, GRPO \cite{shao2024deepseekmath} estimates baselines on-the-fly by standardizing rewards across a group of sampled responses, enabling robust, critic-free updates within the PPO framework.
Recent variants extend this framework through improved advantage estimation, refined clipping, dynamic sampling, or importance-weighting designs, including REINFORCE++~\cite{hu2025reinforce++}, Dr.GRPO~\cite{liu2025understanding}, OPO~\cite{hao2025policy}, DAPO~\cite{yu2025dapo}, StableReinforce~\cite{zhang2025r1}, CISPO~\cite{chen2025minimax}, GSPO~\cite{zheng2025group}, GEPO~\cite{zhang2025gepo} and MaxRL~\cite{tajwar2026maximum}. Another line explores off-policy or experience-enhanced formulations, such as GRPO with expert demonstrations~\cite{yan2025learning}, ExGRPO~\cite{zhan2025exgrpo}, RL-Plus~\cite{dong2025countering}, TOPR~\cite{letapered}, and Asymmetric REINFORCE~\cite{arnal2025asymmetric}. Despite these advances, most critic-free methods still assign the same trajectory-level advantage uniformly to all tokens, leaving fine-grained temporal credit assignment largely underexplored.

\noindent\textbf{Credit Assignment}. Advantage shaping strategies have emerged to facilitate non-uniform credit assignment, utilizing entropy \cite{cheng2025reasoning, wang2025emergent} or statistical methods such as Fisher’s exact test and information gain \cite{sun2025ktae} to modulate token reward. However, these methods are typically heuristic. Another line of work utilizes process reward models to offer denser reward signals \cite{li2025process, cheng2025stop}. Nevertheless, \cite{jia2025do} theoretically suggests that such dense supervision is not statistically superior to outcome supervision. Alternatively, vine sampling~\cite{kazemnejad2025vineppo} has been employed for finer credit assignment. However, this strategy involves an intricate rollout process that can be practically cumbersome to deploy in large-scale training pipelines. Within a parallel strand of research, a promising value-based RL method, TBRM \cite{yuan2025trajectory}, offers a principled alternative, ensuring that credit assignment is implicitly and provably handled during the learning procedure. Despite the theoretical appeal of such a method, our study specifically focuses on optimizing policy gradient algorithms, which currently dominate the landscape of RLVR.

\noindent\textbf{Eligibility Traces}. Eligibility traces \cite{sutton1988learning,sutton1998reinforcement,van2021expected,gupta2024past} serve as a cornerstone of temporal credit assignment in classical RL, significantly accelerating learning efficiency. However, their application has predominantly been confined to value function estimation \cite{10.5555/645529.658134,van2021expected, gupta2024past,elelimy2025deep}. Different from these works, we resurrect the actor-only eligibility trace paradigm \cite{10.5555/645527.657471}. While conceptually echoing the integration of eligibility traces in GRPO($\lambda$) \cite{parthasarathi2025grpo} into critic-free policy gradients, the proposed S-trace method diverges by offering more nuanced algorithmic design.

\section{Preliminaries}
We begin by establishing the formal framework for RLVR, followed by a review of two popular reinforcement learning algorithms commonly used in LLM fine-tuning. Formally, let $\mathbf{o}$ represent the full output sequence and $o_t$ a single token at time step $t$. We formulate  RLVR as a sequential decision-making process. Specifically, given data distribution $\mathcal{Q}$ and a query prompt $x\sim\mathcal{Q}$, the generation process is modeled as an MDP, where the state $s_t$ is defined as the full context history $s_t=x\oplus \textbf{o}_{<t}$, \textcolor{black}{with $\oplus$ denoting concatenation}, and the action $a_t$ at time $t$ is to select a token $o_t$ from the vocabulary $\mathcal{V}$. The environmental dynamics are deterministic such that given the current state $s_t$ and action $a_t$, the environment transitions to $s_{t+1} = s_t \oplus a_t = x\oplus \textbf{o}_{<t}\oplus o_t$ with probability 1.

\subsection{Proximal Policy Optimization (PPO)}
PPO \cite{schulman2017proximal} addresses the instability inherent in vanilla policy gradients by enforcing a trust region constraint without computational overhead of second-order optimization. It employs a clipped surrogate objective to restrain the policy update within a safe vicinity of the behavioral policy $\pi_{\theta_{\text{old}}}$. Formally, the objective function to be maximized is formulated as
\begin{multline}
\label{ppo_objective}
\mathcal{J}_{\text{PPO}}(\theta) = \mathbb{E}_{x\sim\mathcal{Q},\textbf{o}\sim\pi_{\theta_{\text{old}}}(\cdot|x)} \Bigg[\sum_{t=1}^{|\textbf{o}|} \\
\min\Big(r_{t}(\theta)\hat{A}_{t}, \text{clip}(r_{t}(\theta), 1-\epsilon, 1+\epsilon)\hat{A}_t\Big) \Bigg]
\end{multline}
where $r_t(\theta)=\frac{\pi_\theta(o_t|x,\textbf{o}_{<t})}{\pi_{\theta_{\text{old}}}(o_t|x,\textbf{o}_{<t})}$ denotes the importance ratio, $\epsilon$ is the clipping range, and the advantage $\hat{A}_t$ is estimated via Generalized Advantage Estimation (GAE) \cite{schulman2015high} by leveraging immediate reward $r(s_t, a_t)$ and a fitted critic network $V(\cdot)$:
\begin{equation}
\label{gae_formula}
\hat{A}_t=\sum_{k=0}^\infty (\gamma\lambda)^k \delta_{t+k}, 
\end{equation}
where $\delta_t=r(s_t,a_t)+\gamma V(s_{t+1})-V(s_t)$.

\subsection{Group Relative Policy Optimization}
Building upon PPO while streamlining the architecture, GRPO \cite{shao2024deepseekmath} eliminates the critic network $V(\cdot)$ to reduce computational overhead, estimating the baseline via group-wise statistics instead. For each query $x$, GRPO samples a group of $G$ outputs $\{\textbf{o}_i\}_{i=1}^G$ from the old policy $\pi_{\theta_{\text{old}}}$. Then the advantage $\hat{A}_{i,t}$ is computed by standardizing the outcome rewards within the group as
\begin{equation}
\label{group_adv_formula}
\hat{A}_{i,t}=\frac{r(x,\textbf{o}_i)-\text{mean}(\{r(x,\textbf{o}_j)\}_{j=1}^G)}{\text{std}(\{r(x,\textbf{o}_j)\}_{j=1}^G)}.
\end{equation}
The final objective of GRPO combines the clipped surrogate loss in PPO with a KL-divergence penalty, typically approximated via the k3 estimator \cite{shao2024deepseekmath}, to strictly regulate the policy update as

\begin{equation}
\label{grpo_objetive}
\begin{aligned}
&\mathcal{J}_{\text{GRPO}}(\theta)
=\mathbb{E}_{\substack{
    x\sim\mathcal{Q} \\
    \{\textbf{o}_i\}_{i=1}^{G}
    \sim\pi_{\theta_{\text{old}}}(\cdot\mid x)
}}
\Bigg[
\frac{1}{G}\sum_{i=1}^{G}
\frac{1}{|\textbf{o}_i|}
\sum_{t=1}^{|\textbf{o}_i|}
\\[-2pt]
&\quad
\min\Big(
    r_{i,t}(\theta)\hat{A}_{i,t},
    \operatorname{clip}\big(
        r_{i,t}(\theta),
        1-\epsilon,1+\epsilon
    \big)\hat{A}_{i,t}
\Big)
\Bigg].
\end{aligned}
\end{equation}

where $r_{i,t}(\theta)=\frac{\pi_\theta(o_{i,t}|x,\textbf{\textbf{o}}_{i,<t})}{\pi_{\theta_{\text{old}}}(o_{i,t}|x,\textbf{\textbf{o}}_{i,<t})}$.

\section{Methodology}
In this section, we introduce S-trace, which recovers successor-derived credit through importance reweighting and propagates it to selected historical tokens without altering their direct GRPO gradient contributions.
We first derive the successor importance-reweighted formulation and then present the selective mechanism.

\subsection{Reformulating PPO Gradient via Eligibility Traces}
Recognizing that GAE \cite{schulman2015high} is structurally isomorphic to the $\text{TD}(\lambda)$ return mechanism \cite{sutton1988learning} applied to the advantage estimation, we derive Proposition \ref{prop:ppo_grad}, which explicitly reformulates the PPO policy gradient as actor eligibility traces. Proposition \ref{prop:ppo_grad} establishes the off-policy counterpart to Theorem 1 in \cite{parthasarathi2025grpo}. We defer the proof to Appendix \ref{appendix:ppo_grad_proof}.

\begin{proposition}[\textcolor{black}{Eligibility Traces Formulation of PPO}]
\label{prop:ppo_grad}
In the absence of the clipping mechanism and the minimization operator, the policy gradient of PPO is equivalent to the sum of the products of the policy eligibility traces and the temporal difference error. Formally, we have
\begin{align}
\label{ppo_grad}
\nabla_\theta \mathbb{E}\!_{\substack{x\sim\mathcal{Q} \\ \textbf{o}\sim\pi_{\theta_{\text{old}}}(\cdot|x)}} \left[\sum_{t=1}^{|\textbf{o}|} r_{t}(\theta)\hat{A}_{t} \right]\!
\\
=\!\mathbb{E}_{\substack{x\sim\mathcal{Q} \\ \textbf{o}\sim\pi_{\theta_{\text{old}}}(\cdot|x)}} \Bigg[\sum_{t=1}^{|\textbf{o}|}\delta_{t}e_t\Bigg],
\end{align}
where $e_t$ denotes \textcolor{black}{PPO eligibility traces} at time $t$ given by
\begin{align}
e_t = \sum_{k=1}^{t}(\gamma\lambda)^{k-1}r_{\nu}(\theta)\nabla_\theta\log\pi_\theta(o_{\nu}|x,\textbf{o}_{<\nu})
\\
= r_{t}(\theta)\nabla_\theta\log\pi_\theta(o_t|x,\textbf{o}_{<t}) + \gamma\lambda e_{t-1}
\end{align}
with $\nu=t+1-k,\space t\in \big[1,2,\dots,|\textbf{o}|\big]$ and $e_0$ equals $0$.
\end{proposition}
What Proposition \ref{prop:ppo_grad} essentially reveals is that PPO functions as an eligibility traces method akin to \cite{10.5555/645527.657471}. \textcolor{black}{However, it is distinctively characterized by the finite-horizon nature of the RLVR setting, the reliance on importance weights for off-policy correction and the utilization of $\gamma\lambda$ as the trace decay factor (as opposed to $\gamma$ in \cite{10.5555/645527.657471}).}
Nevertheless, Eq. \eqref{ppo_grad} remains dependent on a critic network for calculating the TD error $\delta_t$. To this end, following the same approximation principle adopted in GRPO, we replace the TD error $\delta_t$ with the trajectory-level group advantage $\hat{A} = \delta_{1}$, for which Lemma 1 in \cite{parthasarathi2025grpo} provides an explicit upper bound. Accordingly, we arrive at the policy gradient estimator for \textcolor{black}{eligibility traces within critic-free RLVR framework as}
\begin{equation}
\label{critic_free_actor_only_et_pg}
\begin{aligned}
\sum_{t=1}^{|\textbf{o}_i|}\hat{A}_{i}\Big[r_{i,t}(\theta)\nabla_\theta\log\pi_\theta(o_{i,t}|x,\textbf{o}_{i,<t})+\gamma\lambda e_{i,t-1}\Big].
\end{aligned}
\end{equation}
Rearranging the summation reveals the underlying gradient structure of the raw weight-style GRPO($\lambda$) \cite{parthasarathi2025grpo}, expressed here using unmodified tokenwise importance ratios.

\subsection{Successor Importance Reweighting}
One limitation of the raw critic-free eligibility trace estimator in Eq.~\eqref{critic_free_actor_only_et_pg} is that PPO clipping can suppress a token's gradient contribution when its importance ratio exceeds the clipping boundary, potentially impairing sample efficiency and posing an impediment to efficient reasoning \cite{chen2025minimax}. To address this limitation, we exploit the temporal dependencies inherent in eligibility traces by approximating the historical importance ratios $r_{i,<t}(\theta)$ with the current ratio $r_{i,t}(\theta)$, which can then be factored out. Under this approximation, Eq.~\eqref{critic_free_actor_only_et_pg} admits the following group-based policy gradient estimator
\begin{equation}
\label{sir_pg}
\begin{aligned}
&\frac{1}{G}\sum_{i=1}^{G}\frac{1}{|\textbf{o}_i|}
\sum_{t=1}^{|\textbf{o}_i|}
r_{i,t}(\theta)\hat{A}_{i,t}\\
&\qquad\times\sum_{j=1}^{t}(\gamma\lambda)^{t-j}
\nabla_\theta\log\pi_\theta
(o_{i,j}\mid x,\textbf{o}_{i,<j}).
\end{aligned}
\end{equation}
We term this mechanism \emph{successor importance reweighting}. Under standard clipped objectives, an importance ratio $r_{i,t}(\theta)$ that crosses the relevant clipping boundary suppresses the gradient contribution of its corresponding token $o_{i,t}$. Nevertheless, a subsequent token $o_{i,t+\Delta}$, whose generation is conditioned on $o_{i,t}$, may retain an unclipped ratio $r_{i,t+\Delta}(\theta) \in (1-\epsilon, 1+\epsilon)$ and hence a nonzero learning signal. By incorporating this successor importance ratio into the eligibility trace of $o_{i,t}$, our approach enables surviving downstream signals to propagate credit to earlier tokens even when their own importance ratios have been clipped.

\begin{theorem}
\label{theorem:sir_upper_bound}
Let $\xi=\gamma\lambda\in[0,1)$ and
$\epsilon\in(0,1)$. Denote by $\mathcal G(\theta)$ and $\mathcal G'(\theta)$ the clipped eligibility trace estimators derived from Eqs.~\eqref{sir_pg} and~\eqref{critic_free_actor_only_et_pg}, respectively, with corresponding trace coefficients $\sum_{t=j}^{|\textbf{o}_i|}\xi^{t-j}\bar r_{i,t}$ and $\bar r_{i,j}\sum_{t=j}^{|\textbf{o}_i|}\xi^{t-j}$,
where $\bar r_{i,t}=m_{i,t}r_{i,t}(\theta)$
and $m_{i,t}$ indicates whether PPO clipping
retains the gradient contribution of token $o_{i,t}$. Suppose $1\leq|\textbf{o}_i|\leq H$ almost
surely, $|\hat A_i|\leq A^\star$,
and
$\|\nabla_\theta\log\pi_\theta
(o_{i,t}\mid x,\textbf{o}_{i,<t})\|_2\leq B$.
Assume that, at every reachable position $t$,
\begin{equation}
\label{eq:sir_kl_assumption}
\begin{aligned}
\mathbb E\Big[
&\mathrm{KL}\big[
\pi_{\theta_{\mathrm{old}}}
(\cdot\mid x,\textbf{o}_{i,<t})
\,\|\\
&\pi_\theta(\cdot\mid x,\textbf{o}_{i,<t})
\big]
\,\big|\,
|\textbf{o}_i|\geq t
\Big]\leq\kappa.
\end{aligned}
\end{equation}
Then, with $\mathsf H_H=\sum_{t=1}^{H}t^{-1}$,
\begin{equation}
\label{eq:sir_kl_bound}
\begin{aligned}
&\mathbb E\Big[
\big\|
\mathcal G(\theta)
-\mathcal G^{'}(\theta)
\big\|_2\Big]\\
&\quad\leq
\frac{2A^\star B\xi}
{\epsilon(1-\xi)}
\sqrt{2\kappa}\,(\mathsf H_H - 1).
\end{aligned}
\end{equation}
\end{theorem}
Theorem~\ref{theorem:sir_upper_bound} bounds the discrepancy introduced by successor importance reweighting, with the bound tightening as the policy approaches the rollout policy or as a smaller $\lambda$ shortens the effective temporal horizon of eligibility traces. However, the bound becomes less restrictive as $\gamma\lambda$ increases, revealing a tension between preserving long-range credit assignment for potentially greater sample efficiency and controlling estimator discrepancy. The proof is provided in Appendix~\ref{appendix:sir_upper_bound_proof}.

\subsection{Selective Successor Eligibility Traces}
However, dense eligibility traces indiscriminately propagate successor importance-weighted gradients to historical tokens, potentially introducing noisy updates. Aligning with a broader philosophy of selective mechanism design~\cite{abbas2020selective,anand2021preferential,DBLP:journals/corr/abs-2202-09699}, we propose \textit{Selective Successor Eligibility Traces} (\textbf{S-trace}) to reconcile this trade-off. 

Let $\xi=\gamma\lambda$ and $\nabla_{i,j}(\theta)=\nabla_\theta\log\pi_\theta(o_{i,j}\mid x,\textbf{o}_{i,<j})$. To formulate a surrogate objective that yields the gradient equal to Eq.~\eqref{sir_pg}, we design the \textit{eligible importance weight} (EIW) as
\begin{equation}
\label{eiw}
\begin{aligned}
\tilde r_{i,t}(\theta)
&=\operatorname{sg}\!\left[\frac{r_{i,t}(\theta)}{r^\lambda_{i,t}(\theta)}\right]
  r^\lambda_{i,t}(\theta),
\end{aligned}
\end{equation}
where $r^\lambda_{i,t}(\theta)=\prod_{j=1}^{t}r_{i,j}(\theta)^{\xi^{t-j}}$ and $\operatorname{sg}$ denotes the stop-gradient operator. EIW remains numerically identical to standard importance weight by harnessing the stop-gradient operator, while distinctively inducing the required eligibility traces structure upon differentiation. Furthermore, EIW can be conceptualized as a sequence-level importance weight analogous to that of GSPO \cite{zheng2025group}, distinguished fundamentally by its incorporation of a recency-based structure featuring temporal decay rather than a uniform distribution

S-trace extends EIW through a binary selection mask $\omega_{i,j}\in\{0,1\}$, which determines whether historical token $o_{i,j}$ receives successor-derived credit. We define its \emph{leave-own-out} eligible importance weight directly as
\begin{align}
\label{lowo_omega_iw}
r^{\lambda,\omega}_{i,t}(\theta)
&=r_{i,t}(\theta)
\prod_{j=1}^{t-1}r_{i,j}(\theta)^{\omega_{i,j}\xi^{t-j}},\\
\label{lowo_eiw}
\tilde r^{\lambda,\omega}_{i,t}(\theta)
&=\operatorname{sg}\!\left[
\frac{r_{i,t}(\theta)}{r^{\lambda,\omega}_{i,t}(\theta)}
\right]r^{\lambda,\omega}_{i,t}(\theta).
\end{align}
Hence, under this formulation, the eligibility trace term in S-trace corresponding to token $o_{i,t}$ takes the form as
\begin{align}
\label{s_trace_token_pg}
&e^\omega_{i,t}
=\omega_{i,t}\nabla_{i,t}(\theta)
+\xi e^\omega_{i,t-1},
\quad e^\omega_{i,0}=0,\\
&\nabla_\theta\tilde r^{\lambda,\omega}_{i,t}(\theta)
=r_{i,t}(\theta)
\big[\nabla_{i,t}(\theta)
+\xi e^\omega_{i,t-1}\big].
\end{align}
Eq.~\eqref{s_trace_token_pg} is related to the method in \cite{DBLP:journals/corr/abs-2202-09699}. The critical difference is that they employ eligibility traces for the critic, whereas we focus on the actor with a leave-own-out formulation.
The decay factor $\lambda$ therefore controls temporal persistence, while the selective rate $\rho\in[0,1]$, specifying the fraction of tokens with $\omega_{i,j}=1$ in each rollout, governs the sparsity of historical credit propagation. By preserving each token's own direct gradient contribution regardless of selection, the leave-own-out formulation in Eq.~\eqref{lowo_eiw} recovers GRPO at $\rho=0$ and the dense EIW in Eq.~\eqref{eiw} at $\rho=1$.

\begin{figure*}[!t]
    \centering
    \begin{subfigure}[t]{0.55\textwidth}
        \centering
        \includegraphics[
            width=\linewidth,
            keepaspectratio,
            trim=0 0 0 0,
            clip
        ]{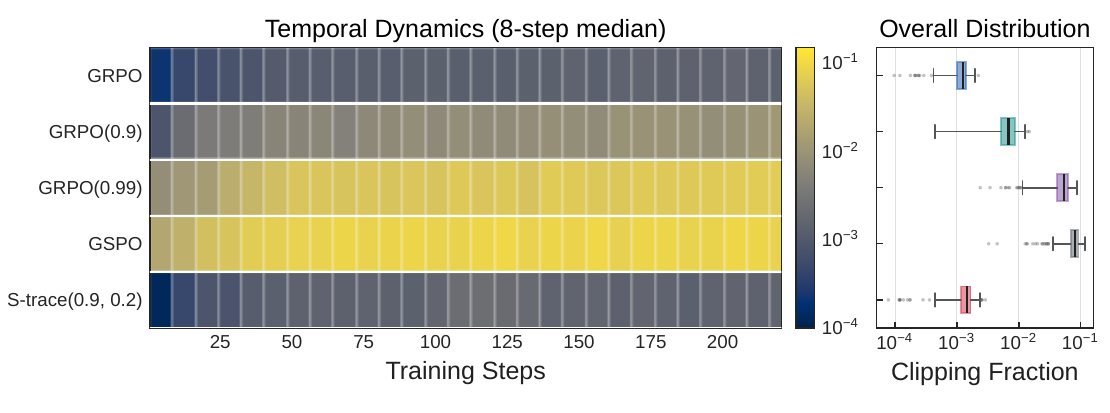}
        \caption{Policy Gradient Clipping Fraction}
        \label{fig:qwen3_4b_clipping}
    \end{subfigure}%
    \hfill%
    \begin{subfigure}[t]{0.21\textwidth}
        \centering
        \includegraphics[
            width=\linewidth,
            keepaspectratio,
            trim=0 0 0 0,
            clip
        ]{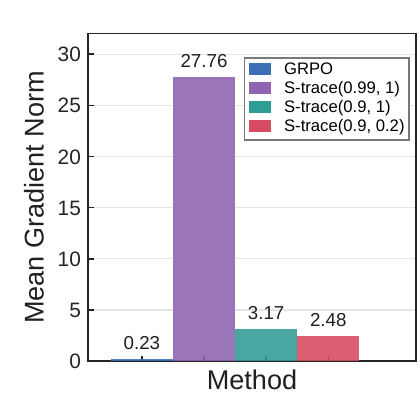}
        \caption{Gradient Norm}
        \label{fig:qwen3_4b_gradnorm}
    \end{subfigure}%
    \hfill%
    \begin{subfigure}[t]{0.21\textwidth}
        \centering
        \includegraphics[
            width=\linewidth,
            keepaspectratio,
            trim=0 0 0 0,
            clip
        ]{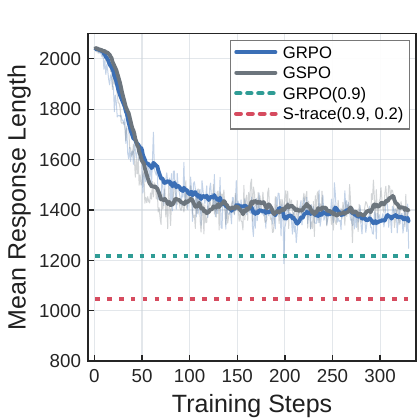}
        \caption{Response Length}
        \label{fig:qwen3_4b_credit_assignment}
    \end{subfigure}%

    \caption{
        Training analysis of Qwen3-4B.
        (a) The left panel shows the clip fraction over time, aggregated using the median within non-overlapping 8-step windows. Colors are normalized on a logarithmic scale. The right panel summarizes the corresponding per-step distributions.
        (b) Gradient norms of GRPO and S-trace variants on Qwen3-4B.
        (c) Evolution of mean response length during Qwen3-4B training. Dashed horizontal lines indicate the minimum mean response lengths attained by the corresponding recency-based methods.
    }
    \label{fig:qwen3_4b_analysis}
\end{figure*}

In the absence of an oracle selection criteria, we set $\omega_{i,t}=0$ for positions within the bottom $1-\rho$ fraction of token entropy in the $i$-th rollout and fix $\rho=0.2$ following the empirical 80/20 observation of~\citet{wang2025beyond}, thereby prioritizing uncertain ``forking'' tokens that may play a pivotal role in the reasoning trajectory.

\section{Experiments}
In this section, we empirically evaluate the proposed S-trace method. We focus on mathematical reasoning tasks. Methodologies validated here are readily transferable to other complex reasoning domains. All experiments are implemented within the veRL \cite{10.1145/3689031.3696075} framework.
\subsection{Experimental Setup}
\textcolor{black}{\textbf{Baselines.}}
We compare our method against the standard GRPO \cite{shao2024deepseekmath} and closely related method GRPO($\lambda$) \cite{parthasarathi2025grpo}. For the latter, we specifically instantiate the trace-style variant and exclude ``both''-type decay and the advantage clamping trick used in \cite{parthasarathi2025grpo} to isolate the core contribution of the eligibility traces. Furthermore, to broaden our evaluation across different optimization paradigms, we include OPO \cite{hao2025policy}, a representative on-policy algorithm integrated into veRL \cite{10.1145/3689031.3696075} training framework that approximates the optimal baseline to minimize policy gradient variance, serving to investigate whether strictly minimized variance necessarily translates to superior performance. Informed by~\cite{parthasarathi2025grpo}, we evaluate GRPO($\lambda$) with $\lambda\in\{0.9,0.99\}$ and instantiate S-trace with $\lambda=0.9$ and $\rho=0.2$.

\begin{table}[!t]
    \centering
    \footnotesize
    \setlength{\tabcolsep}{2pt}
    \renewcommand{\arraystretch}{1.0}
    \caption{
        The best and second-best results are highlighted
        in \textbf{bold} and \underline{underline},
        respectively. Light green denotes GSPO,
        analyzed as a special case of GRPO($\lambda$).
    }
    \label{tab:uniform_vs_recency}

    \newcommand{\mstd}[2]{%
        \mbox{#1\,{\tiny$\pm$#2}}%
    }

    \begin{adjustbox}{max width=\columnwidth}
    \begin{tabular}{@{}lcccccc@{}}
        \toprule
        Method & MATH500 & AIME24 & AIME25
        & AMC23 & Minerva & Avg. \\
        \midrule

        \rowcolor[HTML]{F0F8F1}
        GSPO
        & \mstd{\underline{88.58}}{0.39}
        & \mstd{42.28}{0.82}
        & \mstd{27.37}{1.22}
        & \mstd{92.02}{0.00}
        & \mstd{\underline{38.90}}{0.22}
        & \mstd{57.83}{0.34} \\

        GRPO$(0.9)$
        & \mstd{86.22}{0.04}
        & \mstd{\underline{46.13}}{1.30}
        & \mstd{\underline{34.58}}{0.70}
        & \mstd{\underline{92.28}}{0.00}
        & \mstd{\textbf{43.67}}{0.74}
        & \mstd{\underline{60.57}}{0.24} \\

        \rowcolor[HTML]{EAE8FD}
        S-trace$(0.9,0.2)$
        & \mstd{\textbf{88.92}}{0.08}
        & \mstd{\textbf{56.59}}{0.91}
        & \mstd{\textbf{45.17}}{1.66}
        & \mstd{\textbf{94.35}}{0.00}
        & \mstd{38.27}{0.16}
        & \mstd{\textbf{64.66}}{0.18} \\

        \bottomrule
    \end{tabular}
    \end{adjustbox}
\end{table}

\noindent\textbf{Models and Datasets.}
For our main evaluations, we use the Qwen3 family \cite{yang2025qwen3}, in particular the Qwen3-1.7B and Qwen3-4B variants, as our policy backbones. For training, we utilize DAPO-Math-14k, the English subset of the DAPO-Math-17k dataset \cite{yu2025dapo}, to fine-tune the base policy models. We evaluate generalization performance across five benchmarks, i.e, MATH500 \cite{hendrycks2021measuring}, Minerva \cite{lewkowycz2022solving}, AMC23 \cite{li2024numinamath}, AIME24 and AIME25. Building upon these findings, we highlight the scalability and superior performance of our approach by extending the evaluation to the Qwen3-8B model, incorporating the exceptionally challenging BeyondAIME \cite{seed2025seed1} benchmark. Comprehensive training configurations are provided in Appendix \ref{app:subsec:setup}.

\begin{figure*}[t]
    \centering
    \includegraphics[
        width=\textwidth,
        height=0.30\textheight,
        keepaspectratio,
        trim=0 0 0 0,
        clip,
        angle=0
    ]{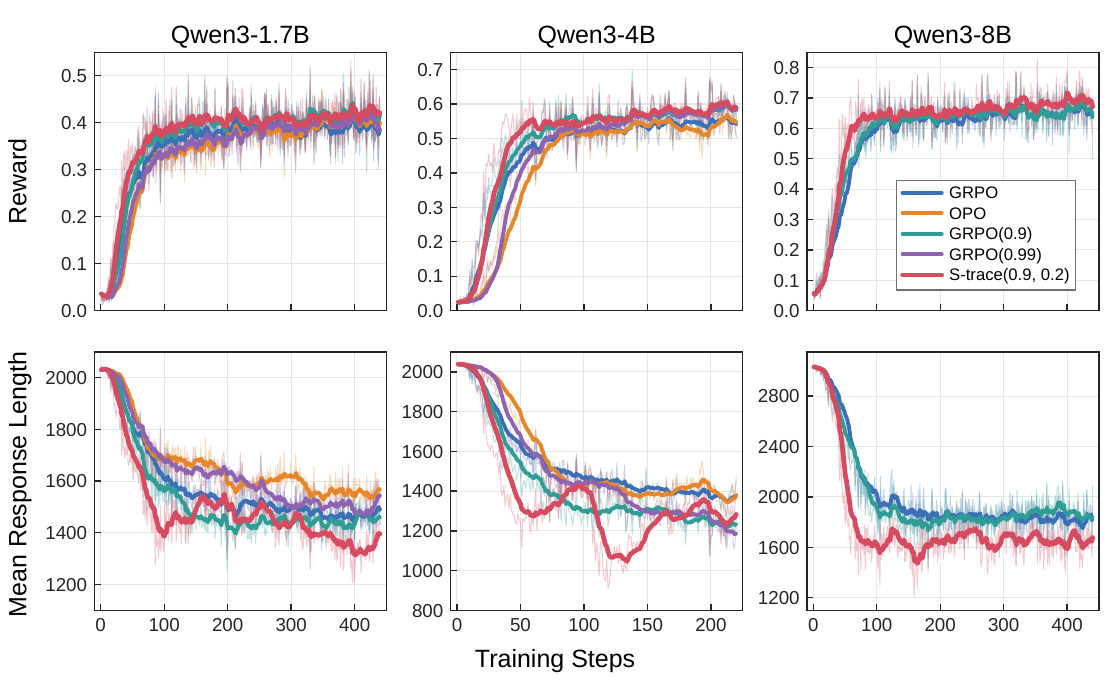}

    \caption{
        Training dynamics of Qwen3 models under different
        training methods.
    }
    \label{fig:qwen3_training}
\end{figure*}

\begin{table*}[!t]
    \centering
    \footnotesize
    \setlength{\tabcolsep}{3pt}
    \renewcommand{\arraystretch}{0.9}
    \setlength{\aboverulesep}{1pt}
    \setlength{\belowrulesep}{1pt}

    \caption{
        Pass@16 results (\%) on five mathematical reasoning
        benchmarks for Qwen3-1.7B and Qwen3-4B
        (mean $\pm$ standard deviation over three runs).
        Best and second-best results for each model
        are highlighted in \textbf{bold} and
        \underline{underline}, respectively.
    }
    \label{tab:main_results}

    \resizebox{\textwidth}{!}{%
    \begin{tabular}{@{}llcccccc@{}}
        \toprule
        Method & Model & MATH500 & AIME24 & AIME25
        & AMC23 & Minerva & Avg. \\
        \midrule

        \multirow{2}{*}{OPO \cite{hao2025policy}}
        & 1.7B
        & 83.73 $\pm$ 0.59
        & 32.52 $\pm$ 0.11
        & 25.73 $\pm$ 0.47
        & 80.81 $\pm$ 0.00
        & 36.19 $\pm$ 0.35
        & 51.80 $\pm$ 0.21 \\
        & 4B
        & 86.76 $\pm$ 0.18
        & 42.26 $\pm$ 0.28
        & 30.97 $\pm$ 2.92
        & 92.11 $\pm$ 0.00
        & 35.45 $\pm$ 0.71
        & 57.51 $\pm$ 0.65 \\

        \addlinespace[1pt]

        \multirow{2}{*}{GRPO \cite{shao2024deepseekmath}}
        & 1.7B
        & 85.22 $\pm$ 0.25
        & \underline{33.93 $\pm$ 1.55}
        & 27.51 $\pm$ 1.37
        & 85.19 $\pm$ 0.00
        & \underline{37.37 $\pm$ 0.38}
        & \underline{53.84 $\pm$ 0.19} \\
        & 4B
        & 86.26 $\pm$ 0.13
        & 39.47 $\pm$ 1.23
        & 31.91 $\pm$ 0.07
        & 91.09 $\pm$ 0.00
        & \textbf{43.72 $\pm$ 0.74}
        & 58.49 $\pm$ 0.15 \\

        \addlinespace[1pt]

        \multirow{2}{*}{GRPO$(0.9)$ \cite{parthasarathi2025grpo}}
        & 1.7B
        & \underline{85.59 $\pm$ 0.41}
        & 33.15 $\pm$ 2.43
        & \underline{28.99 $\pm$ 1.69}
        & 83.78 $\pm$ 0.00
        & 35.08 $\pm$ 0.57
        & 53.32 $\pm$ 0.74 \\
        & 4B
        & 86.22 $\pm$ 0.04
        & 46.13 $\pm$ 1.30
        & 34.58 $\pm$ 0.70
        & \underline{92.28 $\pm$ 0.00}
        & \underline{43.67 $\pm$ 0.74}
        & 60.57 $\pm$ 0.24 \\

        \addlinespace[1pt]

        \multirow{2}{*}{GRPO$(0.99)$ \cite{parthasarathi2025grpo}}
        & 1.7B
        & 84.03 $\pm$ 0.06
        & 28.08 $\pm$ 0.42
        & 25.35 $\pm$ 0.57
        & \textbf{87.96 $\pm$ 0.00}
        & 36.06 $\pm$ 0.46
        & 52.30 $\pm$ 0.19 \\
        & 4B
        & \textbf{90.02 $\pm$ 0.64}
        & \underline{52.96 $\pm$ 2.36}
        & \textbf{45.94 $\pm$ 1.02}
        & 91.50 $\pm$ 0.00
        & 38.52 $\pm$ 0.14
        & \underline{63.79 $\pm$ 0.52} \\

        \addlinespace[1pt]
        \midrule

        \rowcolor[HTML]{EAE8FD}
        & 1.7B
        & \textbf{86.55 $\pm$ 0.01}
        & \textbf{37.21 $\pm$ 1.12}
        & \textbf{29.24 $\pm$ 1.43}
        & \underline{87.54 $\pm$ 0.00}
        & \textbf{37.67 $\pm$ 0.44}
        & \textbf{55.64 $\pm$ 0.24} \\

        \rowcolor[HTML]{EAE8FD}
        \multirow{-2}{*}{S-trace$(0.9, 0.2)$ (ours)}
        & 4B
        & \underline{88.92 $\pm$ 0.08}
        & \textbf{56.59 $\pm$ 0.91}
        & \underline{45.17 $\pm$ 1.66}
        & \textbf{94.35 $\pm$ 0.00}
        & 38.27 $\pm$ 0.16
        & \textbf{64.66 $\pm$ 0.18} \\

        \bottomrule
    \end{tabular}%
    }
\end{table*}

\subsection{Results}
\textbf{Main Results}. Results for Qwen3-1.7B and Qwen3-4B are summarized in Tables \ref{tab:main_results}, while their training reward dynamics are visualized in Figure \ref{fig:qwen3_training}. Across both settings, S-trace achieves greater sample efficiency and superior asymptotic performance, while exhibiting the fastest reduction in mean response length. This suggests that S-trace may enable the policy to identify successful reasoning trajectories earlier and concentrate learning on consequential positions, thereby reducing redundant generation through more effective credit assignment. In contrast, although GRPO($\lambda$) does not outperform GRPO overall, its mean response length declines more rapidly, suggesting that recency-based credit assignment may offer greater token efficiency than uniform credit assignment. We hypothesize that this tendency reflects a structural asymmetry in which noisy tokens accumulate disproportionately toward the end of reasoning trajectories. Whereas uniform credit assignment reinforces all positions with the same outcome-level reward, recency-based methods such as S-trace and GRPO($\lambda$) may place greater cumulative emphasis on earlier tokens. As policy optimization sharpens these tokens' distribution, their influence on subsequent generation may increasingly suppress noisy continuations, an effect that S-trace could further amplify by selectively concentrating successor-derived credit on high-entropy positions.

\textbf{Scaling Properties}. To further explore the scalability of S-trace, we evaluate the fine-tuning performance of GRPO, GRPO($0.9$), and S-trace($0.9,0.2$) on Qwen3-8B. The training dynamics and benchmark results are presented in Figure~\ref{fig:qwen3_training} and Table~\ref{tab:qwen3_8b_results}. S-trace maintains superior asymptotic performance and achieves the highest average accuracy across three benchmarks with shorter responses, demonstrating remarkable token efficiency.

\begin{table}[t]
    \centering
    \footnotesize
    \setlength{\tabcolsep}{2.5pt}
    \renewcommand{\arraystretch}{1.1}
    \caption{
        Performance of different training methods
        on Qwen3-8B across three challenging mathematical reasoning
        benchmarks (\%).
    }
    \label{tab:qwen3_8b_results}

    \begin{adjustbox}{max width=\columnwidth}
    \begin{tabular}{@{}lcccc@{}}
        \toprule
        Qwen3-8B & AIME24 & AIME25 & BeyondAIME & Avg. \\
        \midrule

        GRPO
        & \textbf{69.58}
        & 41.79
        & 28.52
        & \underline{46.63} \\

        GRPO$(0.9)$
        & 57.88
        & \underline{46.82}
        & \underline{31.07}
        & 45.26 \\

        \rowcolor[HTML]{EAE8FD}
        S-trace$(0.9, 0.2)$ (ours)
        & \underline{58.61}
        & \textbf{56.55}
        & \textbf{32.51}
        & \textbf{49.22} \\

        \bottomrule
    \end{tabular}
    \end{adjustbox}
\end{table}



\subsection{Uniform vs. Recency}
Does recency-based credit assignment yield more accurate and concise mathematical reasoning than uniform credit assignment? To investigate this question, we additionally evaluate GSPO~\cite{zheng2025group}, whose unclipped policy gradient admits a uniform eligibility-trace interpretation, as derived in Appendix~\ref{appendix:gspo_uniform_et}. The corresponding results and mean response lengths are presented in Table~\ref{tab:uniform_vs_recency} and Figure~\ref{fig:qwen3_4b_credit_assignment}, respectively. On Qwen3-4B, GRPO($\lambda$) and S-trace outperform GSPO in average pass@16 by absolute margins of 2.74\% and 6.83\%, respectively. Moreover, GSPO's mean response length curve largely overlaps with that of GRPO and even rises above it in later training, whereas both recency-based methods attain substantially lower minimum response lengths, with S-trace achieving the lowest. These findings suggest that capturing temporal dependencies through recency-based eligibility traces promote token-efficient reasoning, while S-trace's further gains  reflect the benefits of selective successor credit propagation.

\subsection{Analysis of Eligible Importance Weight}
We monitor the clip fraction, defined as the average proportion of tokens within reasoning responses that trigger the clipping operator during policy gradient computation. As shown in Figure \ref{fig:qwen3_4b_clipping}, the clip fraction of GRPO($\lambda$) and GSPO consistently exceeds that of both S-trace and GRPO throughout the training process. This observation suggests that sequence-level importance weights, as employed by trace-style GRPO($\lambda$) and GSPO, are more susceptible to clipping than the tokenwise importance weights of GRPO and S-trace, and that the eligible importance weights employed in S-trace closely mirror the numerical behavior of standard importance weights in GRPO for most training steps. We posit that this property is a primary driver of the learning efficiency of S-trace. Specifically, the use of eligible importance weights results in significantly fewer gradients being clipped during policy updates. Moreover, successor importance reweighting allows these surviving signals to propagate to selected historical tokens, including those whose direct gradient contributions are suppressed by clipping, thereby potentially improving gradient utilization and sample efficiency.

\subsection{Ablation and Hyperparameters Analysis}
We investigate how selective credit assignment and temporal persistence affect the performance and optimization properties of S-trace on Qwen3-4B. To assess the effectiveness of entropy-based selection, we compare it against random selection, with results presented in Figure~\ref{fig:qwen3_4b_entropy_ablation}. Entropy-based selection improves average pass@16 from 61.06\% under random selection to 64.66\%, with particularly pronounced gains on AIME24 and AIME25, suggesting that token entropy provides a useful proxy for identifying historical positions that benefit from successor-derived credit.

As shown in Table~\ref{tab:strace_lambda_ablation}, among dense variants with $\rho=1$, S-trace with $\lambda=0.9$ outperforms those with $\lambda\in\{0.7,0.8\}$ on both AIME benchmarks and ultimately achieves the shortest responses, highlighting the benefits of long-horizon credit assignment. However, Figure~\ref{fig:qwen3_4b_gradnorm} shows that dense eligibility traces are accompanied by substantially larger gradient norms, reaching 27.76 and 3.17 for S-trace(0.99, 1) and S-trace(0.9, 1), respectively, both of which exhibit unstable training dynamics in our experiments, compared with 2.48 for selective S-trace(0.9, 0.2). These observations suggest that selective credit assignment may improve optimization stability by moderating gradient magnitudes without shortening the temporal horizon of retained eligibility traces.






\begin{table}[t]
    \centering
    \footnotesize
    \setlength{\tabcolsep}{6pt}
    \renewcommand{\arraystretch}{1.05}
    \caption{
    Performance of S-trace with different
    hyperparameter settings on Qwen3-4B across
    AIME24 and AIME25 after one training epoch.
    Results are reported as mean $\pm$ standard
    deviation over three runs (\%).
    }
    \label{tab:strace_lambda_ablation}

    \begin{tabular}{@{}lcc@{}}
        \toprule
        Qwen3-4B & AIME24 & AIME25 \\
        \midrule
        + S-trace$(0.9, 1)$
        & 51.58 $\pm$ 0.77
        & 37.82 $\pm$ 0.69 \\

        + S-trace$(0.8, 1)$
        & 44.00 $\pm$ 1.03
        & 34.51 $\pm$ 1.63 \\

        + S-trace$(0.7, 1)$
        & 44.38 $\pm$ 4.55
        & 32.83 $\pm$ 2.69 \\
        \bottomrule
    \end{tabular}
\end{table}



\begin{figure}[t]
    \centering
    \begin{subfigure}[t]{0.48\columnwidth}
        \centering
        \includegraphics[
            width=\linewidth,
            keepaspectratio,
            trim=0 0 0 0,
            clip
        ]{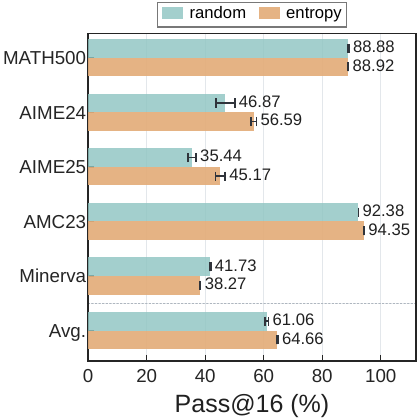}
        \caption{Ablation}
        \label{fig:qwen3_4b_entropy_ablation}
    \end{subfigure}%
    \hfill%
    \begin{subfigure}[t]{0.48\columnwidth}
        \centering
        \includegraphics[
            width=\linewidth,
            keepaspectratio,
            trim=0 0 0 0,
            clip
        ]{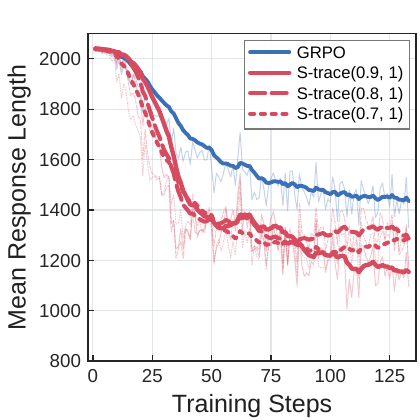}
        \caption{Response Length}
        \label{fig:qwen3_4b_lambda}
    \end{subfigure}%
    \caption{
        (a) Comparison of entropy-based and random selection on Qwen3-4B.
        (b) Evolution of mean response length under different trace decay factors $\lambda$ of S-trace on Qwen3-4B.
    }
    \label{fig:qwen3_4b_ablation_analysis}
\end{figure}







\section{Conclusion and Future Work}
We introduce S-trace, a critic-free eligibility trace method that effectively propagates surviving downstream learning signals to selected historical tokens. We establish an upper bound on the discrepancy between successor-reweighted and raw eligibility-trace estimators, governed jointly by the policy KL divergence and the trace decay factor $\lambda$, and demonstrate improved mathematical reasoning performance across Qwen3 model scales while producing markedly shorter responses. Our comparison results with GSPO further highlights the potential role of recency in token-efficient reasoning, motivating future work on principled token selection, stronger stability guarantees, and a deeper understanding of how temporal dependencies shape credit assignment.

\section*{Limitations}
Our evaluation is limited to the Qwen3 family of LLMs with up to 8B parameters, leaving generalization to larger models, other architectures, and non-mathematical reasoning tasks unexplored. Moreover, our entropy-based token selection remains heuristic, motivating more principled approaches to selective credit assignment.


\bibliography{custom}

\newpage
\appendix

\section{Experimental Setup and Hyperparameters}
\label{app:subsec:setup}
This section outlines the key hyperparameters and environmental configurations adopted in our experiments as shown in Table~\ref{tab:hyperparameters}.

\begin{table}[t]
    \centering
    \footnotesize
    \setlength{\tabcolsep}{3pt}
    \renewcommand{\arraystretch}{1.05}

    \caption{
        Hyperparameters and experimental configurations
        for Qwen3 models. All GPUs are NVIDIA A100.
    }
    \label{tab:hyperparameters}

    \begin{tabular}{@{}lcc@{}}
        \toprule
        \textbf{Hyperparameter}
        & \textbf{1.7B/4B}
        & \textbf{8B} \\
        \midrule

        \multicolumn{3}{@{}l}{\textbf{Optimization}} \\
        Optimizer & AdamW & AdamW \\
        Learning Rate & $1 \times 10^{-6}$ & $1 \times 10^{-6}$ \\
        Weight Decay & 0.01 & 0.01 \\
        LR Scheduler & Constant & Constant \\
        Batch Size & 128 & 64 \\
        Mini-batch Size & 32 & 16 \\

        \midrule
        \multicolumn{3}{@{}l}{\textbf{Algorithm}} \\
        Maximum Prompt Length & 512 & 512 \\
        Maximum Response Length & 2048 & 3072 \\
        Group Size ($G$) & 5 & 8 \\
        Temperature ($\tau$) & 1.0 & 1.0 \\
        Discount Factor ($\gamma$) & 1.0 & 1.0 \\
        KL Coefficient ($\beta$) & 0.001 & 0.001 \\
        Clipping Range ($\epsilon$) & 0.2 & 0.2 \\
        Selective Rate ($\rho$) & 0.2 & 0.2 \\

        \midrule
        \multicolumn{3}{@{}l}{\textbf{Infrastructure}} \\
        Parallelism Strategy & FSDP2 & FSDP2 \\
        Training Hardware & 2 $\times$ A100 & 4 $\times$ A100 \\

        \bottomrule
    \end{tabular}
\end{table}

\section{Proofs}
\allowdisplaybreaks
\subsection{Proof of Proposition \ref{prop:ppo_grad}}
\label{appendix:ppo_grad_proof}
We adopt a derivation logic analogous to that of \cite{parthasarathi2025grpo}, but extend it to the off-policy correction case. The core insight lies in the decomposition and recombination of the Generalized Advantage Estimation (GAE) components. For brevity, write $\nabla_t(\theta)=\nabla_\theta\log\pi_\theta(o_t\mid x,\textbf{o}_{<t})$, and let $\mathbb E_{\mathrm{old}}$ denote expectation over $x\sim\mathcal Q$ and $\textbf{o}\sim\pi_{\theta_{\mathrm{old}}}(\cdot\mid x)$. We treat the GAE estimates as fixed during actor differentiation and set $\delta_t=0$ for $t>T$.

\begin{proof}
Since the sampled trajectories follow the fixed behavior policy, differentiation passes through the tokenwise ratios, yielding
\begin{equation*}
\nabla_\theta r_t(\theta)
=r_t(\theta)\nabla_t(\theta).
\end{equation*}
Using the truncated GAE expansion $\hat A_t=\sum_{k=0}^{T-t}\xi^k\delta_{t+k}$, we obtain
\begin{align*}
&\nabla_\theta\mathbb E_{\mathrm{old}}\!\left[
\sum_{t=1}^{T}r_t(\theta)\hat A_t\right]\\
&=\mathbb E_{\mathrm{old}}\!\left[
\sum_{t=1}^{T}r_t(\theta)\nabla_t(\theta)
\sum_{k=0}^{T-t}\xi^k\delta_{t+k}\right]\\
&=\mathbb E_{\mathrm{old}}\!\left[
\sum_{t=1}^{T}r_t(\theta)\nabla_t(\theta)
\sum_{s=t}^{T}\xi^{s-t}\delta_s\right]\\
&=\mathbb E_{\mathrm{old}}\!\left[
\sum_{s=1}^{T}\delta_s
\sum_{t=1}^{s}\xi^{s-t}
r_t(\theta)\nabla_t(\theta)\right]\\
&=\mathbb E_{\mathrm{old}}\!\left[
\sum_{s=1}^{T}\delta_s e_s\right],
\end{align*}
where the third equality exchanges the order of the finite sums and
\begin{equation*}
\begin{aligned}
e_s
&=\sum_{t=1}^{s}\xi^{s-t}
r_t(\theta)\nabla_t(\theta)\\
&=r_s(\theta)\nabla_s(\theta)+\xi e_{s-1},
\qquad e_0=0.
\end{aligned}
\end{equation*}
This establishes the claimed identity.
\end{proof}

\subsection{Proof of Theorem~\ref{theorem:sir_upper_bound}}
\label{appendix:sir_upper_bound_proof}
\begin{proof}
Let $T_i=|\textbf{o}_i|$ and denote the per-token
policy gradient by
\[
g_{i,t}=\nabla_\theta\log\pi_\theta
(o_{i,t}\mid x,\textbf{o}_{i,<t}).
\]
We first define the clipping indicator as
\begin{equation}
m_{i,t}=
\begin{cases}
\mathbb I\{r_{i,t}(\theta)<1+\epsilon\},
& \hat A_i>0,\\
\mathbb I\{r_{i,t}(\theta)>1-\epsilon\},
& \hat A_i<0,\\
1,&\hat A_i=0,
\end{cases}
\end{equation}
where clipping boundaries are excluded. Recall
that $\bar r_{i,t}=m_{i,t}r_{i,t}(\theta)$.

By collecting the contributions associated with
each historical token, the two clipped
eligibility trace estimators admit the forms
\begin{equation}
\begin{aligned}
\mathcal G(\theta)
&=\frac{1}{G}\sum_{i=1}^{G}
\frac{\hat A_i}{T_i}
\sum_{j=1}^{T_i}g_{i,j}
\sum_{t=j}^{T_i}\xi^{t-j}\bar r_{i,t},\\
\mathcal G^{'}(\theta)
&=\frac{1}{G}\sum_{i=1}^{G}
\frac{\hat A_i}{T_i}
\sum_{j=1}^{T_i}\bar r_{i,j}g_{i,j}
\sum_{t=j}^{T_i}\xi^{t-j}.
\end{aligned}
\end{equation}
Subtracting these expressions and canceling
the terms corresponding to $t=j$ yields
\begin{equation}
\label{eq:sir_proof_difference}
\begin{aligned}
&\mathcal G(\theta)-\mathcal G^{'}(\theta)\\
&=\frac{1}{G}\sum_{i=1}^{G}
\frac{\hat A_i}{T_i}
\sum_{t=2}^{T_i}\sum_{j=1}^{t-1}
\xi^{t-j}
(\bar r_{i,t}-\bar r_{i,j})g_{i,j}.
\end{aligned}
\end{equation}

We next bound the discrepancy between the
effective importance ratios. When $m_{i,t}=0$,
the clipping condition implies
$|r_{i,t}(\theta)-1|\geq\epsilon$, whereas
$|\bar r_{i,t}-1|=1$. Otherwise,
$\bar r_{i,t}=r_{i,t}(\theta)$.
Since $\epsilon\in(0,1)$, both cases satisfy
\begin{equation}
\label{eq:sir_proof_clipping}
|\bar r_{i,t}-1|
\leq\frac{|r_{i,t}(\theta)-1|}{\epsilon}.
\end{equation}
Consequently, the triangle inequality gives
\begin{equation}
\label{eq:sir_proof_pairwise}
\begin{aligned}
|\bar r_{i,t}-\bar r_{i,j}|
&\leq|\bar r_{i,t}-1|
+|\bar r_{i,j}-1|\\
&\leq\frac{
|r_{i,t}(\theta)-1|
+|r_{i,j}(\theta)-1|
}{\epsilon}.
\end{aligned}
\end{equation}
Let $E_{i,t}=\{|\textbf{o}_i|\geq t\}$ and
$s_{i,t}=x\oplus\textbf{o}_{i,<t}$.
Since reachability is determined by the
preceding generation history, conditioning
on $E_{i,t}$ does not alter the conditional
distribution of the current token, yielding
\begin{equation}
\begin{aligned}
&\Pr(o_{i,t}=a\mid s_{i,t},E_{i,t})\\
&\qquad=
\pi_{\theta_{\mathrm{old}}}(a\mid s_{i,t}).
\end{aligned}
\end{equation}
By the law of iterated expectations and the
definition of the importance ratio, we obtain
\begin{equation}
\begin{aligned}
&\mathbb E[
|r_{i,t}(\theta)-1|
\mid E_{i,t}]\\
&=
\mathbb E\Bigg[
\sum_{a\in\mathcal V}
\pi_{\theta_{\mathrm{old}}}(a\mid s_{i,t})\\
&\qquad\times
\left|
\frac{\pi_\theta(a\mid s_{i,t})}
{\pi_{\theta_{\mathrm{old}}}(a\mid s_{i,t})}
-1
\right|
\,\Bigg|\,E_{i,t}\Bigg]\\
&\leq
\mathbb E\Big[
\big\|
\pi_\theta(\cdot\mid s_{i,t})
-\pi_{\theta_{\mathrm{old}}}
(\cdot\mid s_{i,t})
\big\|_1
\,\Big|\,E_{i,t}\Big].
\end{aligned}
\end{equation}
The last inequality becomes an equality
when the two policies have common full support.

Applying Pinsker's inequality yields
\begin{equation}
\begin{aligned}
&\mathbb E[
|r_{i,t}(\theta)-1|
\mid E_{i,t}]\\
&\leq
\mathbb E\Bigg[
\sqrt{
2\,\mathrm{KL}\Big[
\pi_{\theta_{\mathrm{old}}}
(\cdot\mid s_{i,t})
\,\Big\|\,
\pi_\theta(\cdot\mid s_{i,t})
\Big]
}
\,\Bigg|\,E_{i,t}\Bigg].
\end{aligned}
\end{equation}
Since the square-root function is concave,
Jensen's inequality further gives
\begin{equation}
\begin{aligned}
&\mathbb E[
|r_{i,t}(\theta)-1|
\mid E_{i,t}]\\
&\leq
\sqrt{
2\mathbb E\Big[
\mathrm{KL}\Big[
\pi_{\theta_{\mathrm{old}}}
(\cdot\mid s_{i,t})
\,\Big\|\,
\pi_\theta(\cdot\mid s_{i,t})
\Big]
\,\Big|\,E_{i,t}\Big]
}\\
&\leq\sqrt{2\kappa},
\end{aligned}
\end{equation}
where the final inequality follows from
the assumed conditional KL bound.

Therefore, for every reachable position $t$,
\begin{equation}
\label{eq:sir_proof_reachability}
\begin{aligned}
&\mathbb E\big[
\mathbb I(E_{i,t})
|r_{i,t}(\theta)-1|
\big]\\
&=\Pr(E_{i,t})
\mathbb E[
|r_{i,t}(\theta)-1|\mid E_{i,t}]\\
&\leq\sqrt{2\kappa}.
\end{aligned}
\end{equation}
Moreover, for $j<t$, we have
$E_{i,t}\subseteq E_{i,j}$, and hence
\begin{equation}
\label{eq:sir_proof_historical}
\begin{aligned}
&\mathbb E\big[
\mathbb I(E_{i,t})
|r_{i,j}(\theta)-1|
\big]\\
&\leq\mathbb E\big[
\mathbb I(E_{i,j})
|r_{i,j}(\theta)-1|
\big]\\
&\leq\sqrt{2\kappa}.
\end{aligned}
\end{equation}

Combining Eqs.~\eqref{eq:sir_proof_difference}
and~\eqref{eq:sir_proof_pairwise} with the
bounded advantage and policy gradient
assumptions yields
\begin{equation}
\begin{aligned}
&\mathbb E\Big[
\|\mathcal G(\theta)-\mathcal G^{'}(\theta)\|_2
\Big]\\
&\leq\frac{A^\star B}{G\epsilon}
\sum_{i=1}^{G}
\mathbb E\Bigg[
\frac{1}{T_i}
\sum_{t=2}^{T_i}
\sum_{j=1}^{t-1}\xi^{t-j}\\
&\qquad\times
\big(
|r_{i,t}(\theta)-1|
+|r_{i,j}(\theta)-1|
\big)
\Bigg].
\end{aligned}
\end{equation}

Since $T_i\geq t$ for every term included
in the summation, we have $T_i^{-1}\leq t^{-1}$.
Using $T_i\leq H$, we can thus write
\begin{equation}
\begin{aligned}
&\mathbb E\Big[
\|\mathcal G(\theta)-\mathcal G^{'}(\theta)\|_2
\Big]\\
&\leq\frac{A^\star B}{G\epsilon}
\sum_{i=1}^{G}\sum_{t=2}^{H}
\frac{1}{t}\sum_{j=1}^{t-1}\xi^{t-j}\\
&\quad\times\mathbb E\Big[
\mathbb I(E_{i,t})
\big(
|r_{i,t}(\theta)-1|
+|r_{i,j}(\theta)-1|
\big)
\Big].
\end{aligned}
\end{equation}

Applying Eqs.~\eqref{eq:sir_proof_reachability}
and~\eqref{eq:sir_proof_historical} gives
\begin{equation}
\label{eq:sir_proof_final_step}
\begin{aligned}
&\mathbb E\Big[
\|\mathcal G(\theta)-\mathcal G^{'}(\theta)\|_2
\Big]\\
&\leq
\frac{2A^\star B\sqrt{2\kappa}}{\epsilon}
\sum_{t=2}^{H}\frac{1}{t}
\sum_{j=1}^{t-1}\xi^{t-j}\\
&\leq
\frac{2A^\star B\sqrt{2\kappa}}{\epsilon}
\frac{\xi}{1-\xi}(\mathsf H_H - 1),
\end{aligned}
\end{equation}
where the last inequality follows from
\[
\sum_{j=1}^{t-1}\xi^{t-j}
\leq\frac{\xi}{1-\xi},
\qquad
\mathsf H_H=\sum_{t=1}^{H}t^{-1}.
\]
This establishes the claimed upper bound.

Finally, when $\theta=\theta_{\mathrm{old}}$,
all importance ratios equal one and no
token gradient is suppressed by clipping.
Consequently, $\bar r_{i,t}=1$ for every
sampled token, and the two estimators
coincide exactly.
\end{proof}

\section{S-trace}
\label{app:strace_algo}
In this section, we provide the detailed pseudocode for the proposed S-trace algorithm, which is summarized in Algorithm \ref{alg:selective_ptrace}.
\begin{algorithm*}[t]
  \caption{S-trace($\lambda, \rho$)}
  \label{alg:selective_ptrace}
  \begin{algorithmic}
    \STATE {\bfseries Input:} initial policy model $\pi_\theta$, query dataset $\mathcal{Q}$, hyperparameters $\gamma,\lambda,\epsilon$, seletive rate $\rho$;
    \FOR{$s=1$ {\bfseries to} $S$}
    \STATE Sample a batch $\mathcal{Q}_B$ from $\mathcal{Q}$;
    \STATE For each prompt $x\in\mathcal{Q}_B$, sample $G$ trajectories $\{\textbf{o}_{i}\}_{i=1}^G\sim \pi_{\theta_{\text{old}}}(\cdot|x)$;
    \STATE Compute the advantage $\hat{A}_i$ for each sampled response $\textbf{o}_i$ using Eq. \eqref{group_adv_formula};
    \FOR{$b=1$ {\bfseries to} $B$}
    \STATE Compute the $r_{i,t}^{\lambda,\omega}(\theta)$ for each token $o_{i,t}$:
    \[
    \begin{aligned}
    r^{\lambda,\omega}_{i,t}(\theta)
    &=
    \frac{
        \pi_\theta(o_{i,t}\mid x,\mathbf{o}_{i,<t})
    }{
        \pi_{\theta_{\mathrm{old}}}(o_{i,t}\mid x,\mathbf{o}_{i,<t})
    }\times
    \prod_{j=1}^{t-1}
    \left[
    \frac{
        \pi_\theta(o_{i,j}\mid x,\mathbf{o}_{i,<j})
    }{
        \pi_{\theta_{\mathrm{old}}}(o_{i,j}\mid x,\mathbf{o}_{i,<j})
    }
    \right]^{\omega_{i,j}(\gamma\lambda)^{t-j}}.
    \end{aligned}
    \]
    \STATE Compute the leave-own-out EIW $\tilde{r}^{\lambda,\omega}_{i,t}(\theta)$ for each token $o_{i,t}$ using $r_{i,t}^{\lambda,\omega}(\theta)$ via Eq. \eqref{lowo_eiw};
    \STATE Update policy $\pi_\theta$ by maximizing:
    \[
    \begin{aligned}
        \mathbb{E}_{\substack{x\sim\mathcal{Q} \\ \{\textbf{o}_i\}_{i=1}^G\sim\pi_{\theta_{\text{old}}}(\cdot|x)}}\Bigg[ & \frac{1}{G}\sum_{i=1}^G \frac{1}{|\textbf{o}_i|}\sum_{t=1}^{|\textbf{o}_i|} \min\Big(\tilde{r}_{i,t}(\theta)\hat{A}_{i}, \text{clip}(\tilde{r}_{i,t}(\theta), 1-\epsilon, 1+\epsilon)\hat{A}_{i}\Big)\Bigg]
    \end{aligned}
    \]
    \ENDFOR
    \ENDFOR
  \end{algorithmic}
\end{algorithm*}

\section{A Uniform Eligibility Traces View of GSPO}
\label{appendix:gspo_uniform_et}
By elevating importance ratios and clipping mechanism to the sequence level, GSPO \cite{zheng2025group} strictly aligns the optimization granularity with sequence-level rewards. This structural consistency effectively mitigates the volatility inherent in the token-level importance weights of GRPO, thereby securing the training stability requisite for large-scale architectures, particularly Mixture-of-Experts (MoE) models. Here, we elucidate the theoretical connection between trace-style GRPO($\lambda$) and GSPO in this subsection, characterizing the latter as a specialized instantiation of eligibility traces mechanisms with uniform credit assignment.

Disregarding the clipping mechanism and the minimization operator to focus on the core gradient dynamics, the objective function for trace-style GRPO($\lambda$) can be formulated as
\begin{equation}
\label{app:grpo-lambda-t_obj}
\begin{aligned}
&\mathcal{J}_{\text{GRPO}(\lambda)}(\theta)= \\
&\mathbb{E}_{\substack{x \sim \mathcal{Q}\\\{\textbf{o}_i\}_{i=1}^G \sim \pi_{\theta_{\text{old}}}(\cdot|x)}}
\Biggl[\frac{1}{G}\sum_{i=1}^G\frac{1}{|\textbf{o}_i|}
\sum_{t=1}^{|\textbf{o}_i|}
r^\lambda_{i,t}(\theta) \hat{A}_i\Biggr].
\end{aligned}
\end{equation}
We restrict our attention to the optimization objective associated with the terminal token of each group response $o_i$, denoted as
\begin{equation}
\label{app:last-token_grpo-lambda-t_obj}
\begin{aligned}
&\mathcal{J}_{\text{GRPO}(\lambda), -1}(\theta) \\
&=\mathbb{E}_{\substack{x \sim \mathcal{Q}\\\{\textbf{o}_i\}_{i=1}^G \sim \pi_{\theta_{\text{old}}}(\cdot|x)}}
\Biggl[\frac{1}{G}\sum_{i=1}^Gr^\lambda_{i, |\textbf{o}_i|}(\theta)
\hat{A}_i\Biggr].
\end{aligned}
\end{equation}
Deriving the policy gradient for this terminal-token objective reveals
\begin{equation}
\label{app:last-token_grpo-lambda-t_grad}
\begin{aligned}
&\nabla_\theta \mathcal{J}_{\text{GRPO}(\lambda), -1}(\theta) \\
&= \mathbb{E}_{\substack{x \sim \mathcal{Q} \\
\{\textbf{o}_i\}_{i=1}^G \sim \pi_{\theta_{\text{old}}}}}
\Biggl[ \frac{1}{G} \sum_{i=1}^G \hat{A}_i\nabla_\theta
r^\lambda_{i, |\textbf{o}_i|}(\theta) \Biggr] \\
&= \mathbb{E}_{\substack{x \sim \mathcal{Q} \\
\{\textbf{o}_i\}_{i=1}^G \sim \pi_{\theta_{\text{old}}}}}
\Biggl[ \frac{1}{G} \sum_{i=1}^G \hat{A}_i \\
&\quad \prod_{t=1}^{|\textbf{o}_i|}
\Biggl( \frac{
\pi_{\theta}(o_{i,t}|x,\textbf{o}_{i,<t})^{\textcolor{blue}{(\gamma\lambda)^{|\textbf{o}_i|-t}}}
}{
\pi_{\theta_{\text{old}}}(o_{i,t}|x,\textbf{o}_{i,<t})^{\textcolor{blue}{(\gamma\lambda)^{|\textbf{o}_i|-t}}}
}\Biggr) \\
&\quad \times \sum_{t=1}^{|\textbf{o}_i|}
\textcolor{blue}{(\gamma\lambda)^{|\textbf{o}_i|-t}} \nabla_\theta \log \pi_\theta
(o_{i,t} | x, \textbf{o}_{i,<t}) \Biggr].
\end{aligned}
\end{equation}
Similarly, GSPO optimizes the policy by leveraging a sequence-level importance weight. The original objective function of GSPO is defined as
\begin{equation}
\label{app:gspo_obj}
\begin{aligned}
&\mathcal{J}_{\text{GSPO}}(\theta) \\
&= \mathbb{E}_{\substack{x \sim \mathcal{Q} \\
\{\textbf{o}_i\}_{i=1}^G \sim \pi_{\theta_{\text{old}}}(\cdot|x)}}
\Biggl[ \frac{1}{G} \sum_{i=1}^G s_i(\theta) \hat{A}_i \Biggr],
\end{aligned}
\end{equation}
where the sequence importance weight $s_i(\theta)$ employs a length-normalized sequence likelihood given by
\begin{equation}
\label{seq_iw}
\begin{aligned}
s_i(\theta)
&= \left( \frac{\pi_\theta(\textbf{o}_i|x)}{\pi_{\theta_{\text{old}}}(\textbf{o}_i|x)} \right)^{\frac{1}{|\textbf{o}_i|}} \\
&= \prod_{t=1}^{|\textbf{o}_i|}
\frac{\pi_\theta(o_{i,t} | x, \textbf{o}_{i,<t})^{\frac{1}{|\textbf{o}_i|}}}{\pi_{\theta_{\text{old}}}(o_{i,t} | x,\textbf{o}_{i,<t})^{\frac{1}{|\textbf{o}_i|}}}.
\end{aligned}
\end{equation}
Computing the gradient of Eq. \eqref{app:gspo_obj} yields
\begin{equation}
\label{gspo_grad}
\begin{aligned}
&\nabla_\theta \mathcal{J}_{\text{GSPO}}(\theta) \\
&= \mathbb{E}_{\substack{x \sim \mathcal{Q} \\
\{\textbf{o}_i\}_{i=1}^G \sim \pi_{\theta_{\text{old}}}(\cdot|x)}}
\Biggl[ \frac{1}{G}\sum_{i=1}^G \hat{A}_i \\
&\qquad \frac{s_i(\theta)}{|\textbf{o}_i|} \sum_{t=1}^{|\textbf{o}_i|}
\nabla_\theta \log\pi_\theta
(o_{i,t} | x, \textbf{o}_{i,<t}) \Biggr].
\end{aligned}
\end{equation}
Substituting the explicit expansion of $s_i(\theta)$ from Eq. \eqref{seq_iw} into Eq. \eqref{gspo_grad} allows us to express the GSPO gradient in a eligibility traces-analogous form

\begin{equation}
\label{gspo_trace_grad}
\begin{aligned}
&\nabla_\theta \mathcal{J}_{\text{GSPO}}(\theta) \\
&= \mathbb{E}_{\substack{x \sim \mathcal{Q} \\
\{\textbf{o}_i\}_{i=1}^G \sim \pi_{\theta_{\text{old}}}}}
\Biggl[ \frac{1}{G} \sum_{i=1}^G \hat{A}_i \\
&\quad \prod_{t=1}^{|\textbf{o}_i|}
\Biggl(\frac{
\pi_\theta(o_{i,t} | x,\textbf{o}_{i,<t})^{\textcolor{red}{\frac{1}{|\textbf{o}_i|}}}
}{
\pi_{\theta_{\text{old}}}(o_{i,t} | x, \textbf{o}_{i,<t})^{\textcolor{red}{\frac{1}{|\textbf{o}_i|}}}
}\Biggr) \\
&\quad \times\sum_{t=1}^{|\textbf{o}_i|}
\textcolor{red}{\frac{1}{|\textbf{o}_i|}}\nabla_\theta \log \pi_\theta
(o_{i,t} | x, \textbf{o}_{i,<t}) \Biggr].
\end{aligned}
\end{equation}

This derivation unveils that GSPO inherently functions as a variation of the eligibility traces method where the temporal decay factors are replaced by a uniform averaging constant $1/|\textbf{o}_i|$ and the importance weights are smoothed via a geometric mean. Rooted in this theoretical kinship, empirical evaluations of GSPO reveal a high incidence of clipping events \cite{zheng2025group} resembling  that observed in our GRPO($\lambda$) experiments. Nevertheless, unlike GSPO, the trace-style GRPO($\lambda$) extends the gradient computation beyond the terminal token to leverage the partial sequences corresponding to every token within the response, synthesizing these partial signals through arithmetic averaging.

\end{document}